\documentclass[letterpaper, 10 pt]{IEEEtran} 
\usepackage{amsmath,graphicx}

\usepackage{amsthm}
\usepackage{amssymb}
\usepackage[all]{xy}
\theoremstyle{plain}

\newtheorem{thm}{Theorem}
\newtheorem{prob}{Problem}
\newtheorem{prop}[thm]{Proposition}

\newtheorem{lem}[thm]{Lemma}

\theoremstyle{definition}
\newtheorem{df}[thm]{Definition}

\newtheorem{eg}[thm]{Example}

\def\rank{\textrm{rank }}
\def\pr{\textrm{pr}}
\def\id{\textrm{id}}
\def\image{\textrm{image }}
\def\col{\mathcal}

\title{A Topological Lowpass Filter for Quasiperiodic Signals}

\author{Michael Robinson$^{(1)}$~\IEEEmembership{Member, IEEE}
\thanks{(1) Mathematics and Statistics, American University, Washington, DC, USA email: michaelr@american.edu}}

\begin{document}
\maketitle

\begin{abstract}
This article presents a two-stage topological algorithm for recovering an estimate of a quasiperiodic function from a set of noisy measurements.  The first stage of the algorithm is a topological phase estimator, which detects the quasiperiodic structure of the function without placing additional restrictions on the function.  By respecting this phase estimate, the algorithm avoids creating distortion even when it uses a large number of samples for the estimate of the function.
\end{abstract}


\section{Introduction}

This article addresses the problem of recovering a class of signals with periodic-like structure that are masked both by noise and by a general warping of the domain.  Additive noise is mitigated by averaging over groups of samples, but this requires care to preserve structure of the signal.  If the signal has a definite spectral shape, then a matched linear filter has the optimal weights for the samples to be averaged.  If the signal does not have a definite spectral shape -- for instance, if it is subject to an unknown time warping -- linear matched filters do not exist.  This article presents a novel, two-stage adaptive filter for signals that are subjected to unknown warping of the domain, which may be a general smooth manifold.  We call this filter the \emph{quasiperiodic low pass filter (QPLPF)}.

\subsection{Historical context}

Although \emph{almost periodic} signals -- those within a certain metric distance of a periodic signal -- are a natural generalization beyond periodic signals, they do not accurately represent signals that are periodic under a warped timescale.  These kind of signals are common in music processing \cite{Mueller_2015}.  If the domain has two or more dimensions, then many more possibilities for warping arise.  The path to greater generality is embodied in the two dimensional images captured by cryo-electron microscopy.  These images have a different underlying symmetry group -- the group of rotations in $\mathbb{R}^3$ -- and the smooth structure of this group can be exploited to great effect \cite{Singer_2013}.

Adaptive filters are often used in image processing (for instance \cite{Dabov_2007}, among many others), but ignoring internal structure of the signal can lead to poor results \cite{Willett_2014}.  Class averaging \cite{Singer_2011,vanHeel_1984} is usually presented as a way to ensure that this structure is preserved, but theoretical guarantees are usually given for a specific problem domain.  The QPLPF we present in this paper is a \emph{general} class averaging filter, and is applicable to many problem domains.  To support the broad application of the QPLPF, we impose only weak theoretical constraints on the input signals.  Under these constraints we obtain surprisingly strong theoretical guarantees.

Signals that have a hidden state space are identifiable using the topology of \emph{delay embeddings} \cite{DeSilva_2011}, a concept that can be traced to a paper by Takens \cite{Takens_1981}.  Many papers have discussed ways to find the hidden state of a dynamical system; recovering the \emph{phase space} from measurements \cite{Chelidze_2008,Chelidze_2006,Casdagli_1991, Sauer_1991,Hegger_2002}.  The key theoretical guarantees arise from transversality results for smooth manifolds.  These can be lifted to geometric conditions for recovering state spaces up to topology under noisy conditions \cite{Chazal_2009,Chazal_2006,Niyogi_2008}.  Although the present paper does not require a complete estimation of a topological space, we obtain similar performance bounds in the face of noise.

\section{Problem statement}
We begin by specifying the class of signals of interest: those with nontrivial \emph{quasiperiodic factorizations}.

\begin{df} \cite{RobinsonSampTA2015}
A function $u:M\to N$ from one smooth manifold $M$ to another $N$ is called \emph{$(\phi,U)$-quasiperiodic} if there exists another smooth manifold $C$, a smooth function $U:C \to N$, and a surjective submersion $\phi: M \to C$ such that $u=U \circ \phi$.  We say $u$ \emph{factors through} $\phi$ and call $C$ the \emph{phase space}. 
\end{df}

Quasiperiodic functions are a strict generalization of dynamically time warped functions, in which the phase function $\phi:\mathbb{R}\to\mathbb{R}$ is a diffeomorphism.  We treat dynamic time warping experimentally in Section \ref{sec:results}, though our algorithm works for all quasiperiodic functions as shown by Theorem \ref{thm:phase_recovery} (noisless case) and in Section \ref{sec:noise_perf} (noisy case).  Although our simulated data is rather simplistic, we note that Theorem \ref{thm:phase_recovery} establishes a substantially more general condition for class averaging.

The main problem addressed by this article is the following:
\begin{prob}
  Assume the following:
  \begin{enumerate}
  \item $M$ is a finite dimensional manifold,
  \item $N$ is a finite dimensional vector space,
  \item $n$ is a random field $M \to N$ whose values are identically distributed and independent from one another, and
  \item $M$ is acted upon transitively by a group $G$ of diffeomorphisms.
  \end{enumerate}
  Given a function $\tilde{u}(x)=u(x)+n(x)$ consisting of the sum of a $(\phi,U)$-quasiperiodic function $u:M\to N$ and a noise signal $n:M\to N$, recover an estimate of $u$.  We will assume that $\tilde{u}$ is only specified at a discrete set of points $X=\{x_i\} \subset M$.
\end{prob}

\section{Algorithm description}

The \emph{quasiperiodic low pass filter} (QPLPF) estimates $u$ from samples of $\tilde{u}$ and is tuned by several parameters:
\begin{enumerate}
\item The \emph{delays} $g_1, \dotsc, g_m \subset G$, and
\item The \emph{neighborhood size} $S$, which is a positive integer.
\end{enumerate}
The QPLPF consists of two distinct stages:
\begin{enumerate}
\item \emph{Topological estimation,} a discrete estimation of the phase function $\phi$.  This stage consists of two steps:
  \begin{enumerate}
  \item \emph{Delay immersion,} constructing an auxillary phase function $F: M \to N^{m+1}$
    \begin{equation*}
      F(x) = \left(\tilde{u}(x),\tilde{u}(g_1 x), \dotsc, \tilde{u}(g_m x)\right),
    \end{equation*}
    using a fixed set $\{g_1, \dotsc, g_m\} \subset G$ of group elements to translate copies of $\tilde{u}$.
  \item \emph{Discretization,} which extracts a distance-based graph $H$ using the set $X \subseteq M$ as vertices based on the image of $F$.  Since $N$ is a normed vector space, we can select a metric $d$ on $N^{m+1}$.  For a given $x\in X$, its set of adjacent edges in $H$ is defined to be the $S$ nearest neighbors\footnote{If there are more than $S$ nearest neighbhors, then the adjacent edges are drawn arbitrarily from this set. To simplify the notation we assert that each vertex is adjacent to itself, but that this does not count against the $S$ nearest neighbors.} measured via $d(F(x),F(y))$.  
  \end{enumerate}
\item \emph{Neighborhood averaging,} a statistical estimator for $U$ using the neighborhoods of $H$:
  \begin{equation}
    \label{eq:qplpf_highlevel}
    (\text{QPLPF}\; \tilde{u})(x_i) = \frac{1}{1+ S}\left(\sum_{[x_i,x_j]\in H} \tilde{u}(x_j)\right).
  \end{equation}
\end{enumerate}

\section{Theoretical discussion}

Quasiperiodic factorizations of smooth functions have a number of interesting properties that make them both expressive and useful models of signals.

\begin{eg}
Every smooth function $u:M\to N$ has a \emph{trivial quasiperiodic factorization}, namely $(\id_M, u)$, where $\id_M:M\to M$ is the identity function.  The QPLPF filter reduces to a sliding window average on functions that have \emph{only} the trivial factorization.
\end{eg}

\begin{eg}
Consider the phase modulated sinusoid $u(t) = \sin \left(\phi(t)\right)$ for $t\in (-\infty,\infty)$.  If we use $\phi: \mathbb{R} \to S^1$, where $S^1 = \{(x,y) \in \mathbb{R}^2: x^2+y^2=1\}$ is the unit circle and $U: S^1 \to \mathbb{R}$ with $U(x,y) = y$, then $U \circ \phi = u$.  This is a nontrivial quasiperiodic factorization of $u$ if the derivative of $\phi$ is never zero.
\end{eg}

\begin{prop}
If a smooth function $u$ from a manifold $M$ to a metric space $N$ has a quasiperiodic factorization with a compact phase space, then $u$ is bounded.
\end{prop}
Unbounded smooth functions cannot have $S^1$ as a phase space, for instance.
\begin{proof}
Suppose that $u$ is $(\phi,U)$-quasiperiodic and that the domain of $U:C \to N$ is compact.  The image of $u$ coincides with the image of $U$, which is compact since $U$ is continuous.  Thus this image is closed and bounded, hence $u$ is bounded.
\end{proof}

\begin{prop}
Any compactly supported smooth function from $\mathbb{R}^n \to \mathbb{R}$ is quasiperiodic with phase space $C=S^n$. 
\end{prop}
This might be a wildly uninformative quasiperiodic factorization.  There are usually better ones as Proposition \ref{prop:uqf} states.
\begin{proof}
$S^n$ is the one-point compactification of $\mathbb{R}^n$, formed by adding a point at infinity.  Since $u$ is compactly supported, we merely construct $\phi$ so that a neighborhood of infinity in $S^n$ has zero preimage, and then the complement (which includes the support of $u$) is diffeomorphic to $\mathbb{R}^n$.
\end{proof}


\subsection{Obtaining quasiperiodic factorizations}

To establish the theoretical validity of the QPLPF, we show that if $u$ is $(\phi,U)$-quasiperiodic, then the QPLPF will produce a (possibly less compact) quasiperiodic factorization of $u$ in which $F$ is the phase function. 

\begin{lem}
  \label{lem:toprank}
  Suppose $u:M\to N$ is a smooth function.  If $M$ is a compact manifold that is acted upon transitively by a group $G$ of diffeomorphisms, then there is a finite set $\{g_1, \dotsc, g_m\}\subset G$ for which the function $F: M \to N^{m+1}$ given by
  \begin{equation*}
    F(x) = \left(u(x),u(g_1 x), \dotsc, u(g_m x)\right)
  \end{equation*}
  has constant rank
  \begin{equation*}
    \rank dF(x) = \max_{y\in M} \rank du(y)
  \end{equation*}
  for all $x\in M$.
\end{lem}
\begin{proof}
  Consider the set $R \subseteq M$ given by
  \begin{equation*}
    R = \{x \in M : \rank du(x) = \max_{y\in M} \rank du(y) \}.
  \end{equation*}
  Because $u$ is smooth, it is continuous, so $R$ is open.  Then the collection
  \begin{equation*}
    \col{R} = \{g R : g \in G\}
  \end{equation*}
  is an open cover of $M$ because each $g$ is a diffeomorphism and $G$ acts transitively.  Because $M$ is compact, there is a finite subcollection
  \begin{equation*}
    \col{R}' = \{g_1 R, \dotsc, g_m R \} \subset \col{R}
  \end{equation*}
  that is also an open cover of $M$.  Thus for any $x\in M$, $g_i x \in R$ for at least one of $i = 1, \dotsc, m$.  Thus
  \begin{eqnarray*}
    \rank dF(x) &= &\max\{\rank du(x), \rank du(g_1 x), \dotsc,\\
    && \rank du(g_m x) \} \\
    &=& \rank du(g_i x) \\
    &=& \max_{y\in M} du(y).
  \end{eqnarray*}
\end{proof}

When there is no noise, the topological estimation stage of the QPLPF recovers a quasiperiodic factorization.

\begin{thm}
  \label{thm:phase_recovery}
  Suppose $u:M\to N$ is a smooth function, where $M$ is a compact manifold that is acted upon transitively by a group $G$ of diffeomorphisms.  Using the finite set $\{g_1, \dotsc, g_m\}\subset G$ and the function $F: M \to N^{m+1}$ defined in Lemma \ref{lem:toprank},
  \begin{equation*}
    F(x) = \left(u(x),u(g_1 x), \dotsc, u(g_m x)\right)
  \end{equation*}
  define $C= \image F$.  If $m=0$, then $(F,\id)$ is a quasiperiodic factorization of $u$.  If $m>0$, then
  \begin{enumerate}
  \item $C$ is an immersed submanifold of $N^{m+1}$, let $i:C' \to C \subset N^{m+1}$ be the immersion, and 
  \item $F$ can be pulled back to $\phi : M \to C'$ so that there is a $U:C' \to N$ with $u = U \circ \phi$ being a quasiperiodic factorization.
  \end{enumerate}
\end{thm}
\begin{proof}
  \begin{enumerate}
  \item By Lemma \ref{lem:toprank}, $dF$ can be constructed so that it has constant rank, so $C$ is an immersed submanifold \cite[Thm. 7.13]{Lee_2003}.  Let $i: C' \to C$ be the immersion.  Without loss of generality, assume that self-intersections of $C'$ are transverse.  Self-intersections are therefore finite sets, because they have dimension
    \begin{equation*}
      2 \dim C' - (m+1) \dim N \le (1 - m) \dim C' \le 0
      \end{equation*}
    since $\dim C' \le \dim N$ by construction.
  \item $F$ is surjective onto $C$ by construction, so we wish to construct a surjective $\phi$ so that the diagram
    \begin{equation*}
      \xymatrix{
        M \ar[r]^F \ar[d]_\phi& C\\
        C' \ar[ur]_i & \\
      }
    \end{equation*}
    commutes.  The only issue is when the image of $C'$ intersects itself, because away from those self-intersections, $i$ is injective.  Let $x\in M$ be such that $F(x)$ is at a place where $C'$ intersects itself in $C$.  We assumed self-intersections of $C'$ are transverse, so there are finitely many preimages $y_1,\dotsc, y_p$ of $F(x)$ in $C'$ which could be chosen as $\phi(x)$.   Because $dF$ is of constant rank and because the self-intersections are transverse, $dF$ will take the tangent space at $x\in M$ to exactly one of the images of the tangent spaces $T_{y_j}$ through $i$.  We simply let $\phi(x) = y_j$, and define $U = \pr_1 \circ i$ to obtain the quasiperiodic factorization of $u$.
  \end{enumerate}
\end{proof}

\subsection{Universal quasiperiodic factorizations}
Although there are many quasiperiodic factorizations of a smooth function, they are related to one another.  Although $(F,\pr_1)$ may differ from $(\phi,U)$, its use in the QPLPF will not destroy the structure of $u$.

\begin{df}
  The quasiperiodic factorizations of $u:M \to N$ form a category ${\bf QuasiP}(u)$ in which the objects are quasiperiodic factorizations $(\phi,U)$, the morphisms $(\phi,U)\to(\phi',U')$ are commutative diagrams of the form
  \begin{equation*}
    \xymatrix{
      M \ar[r]^\phi \ar[d]^{\phi'}& C \ar[dl]\ar[d]^U \\
      C' \ar[r]^{U'} & N \\
      }
  \end{equation*}
\end{df}

\begin{eg}
  The category ${\bf QuasiP}(u)$ is usually not finite: consider $u(x) = \sin x$, because then if $\phi:\mathbb{R}\to S^1$, $U$ can represent any finite number of periods of $\sin$ on $S^1$.
\end{eg}

\begin{prop} \cite[Thm. 5]{RobinsonSampTA2015}
\label{prop:uqf}
  If $u:M \to N$ is a smooth map, the category ${\bf QuasiP}(u)$ has a unique final object called the \emph{universal quasiperiodic factorization} of $u$.  It also has a trivial initial object $(\text{id},u)$.  The category ${\bf QuasiP}(u)$ also has coproducts, which allow one to constructively reduce the phase space.
\end{prop}

Quasiperiodic factorizations impose specific restrictions on the ranks of the derivatives of $\phi$ and $U$.

\begin{lem}
  \label{lem:quasirank}
  If $(\phi,U)$ is any quasiperiodic factorization of $u$, then
  \begin{equation*}
    \rank du(x) \le \min\left\{\rank d\phi(x), \rank dU(\phi(x))\right\} \le \rank d\phi.
  \end{equation*}
  and $\rank dU(\phi(x)) = \rank du(x)$ for all $x\in M$.
\end{lem}
\begin{proof}
  Merely note that the $\rank d\phi$ is constant because $\phi$ is a submersion.  Additionally, by Sylvester's inequality, if $\phi : M \to C$,
  \begin{eqnarray*}
    \rank d\phi + \rank dU(\phi(x)) - \dim C& \le& \rank du(x)\\
    \rank dU(\phi(x)) &\le & \rank du(x)\\
  \end{eqnarray*}
  from which the result follows.
\end{proof}

The universal quasiperiodic factorization involves the unique minimal phase space.

\begin{prop}
  \label{prop:uqf_dim}
  If $(\phi,U)$ is a universal quasiperiodic factorization, then
  \begin{equation*}
    \rank d\phi = \max_{y \in M} \rank du(y).
  \end{equation*}
\end{prop}
\begin{proof}
  If it happens that $\rank d\phi > \max_{y \in M} \rank du(y)$, then we can show the factorization is not universal.  Specifically, notice that by Lemma \ref{lem:quasirank}
  \begin{equation*}
    \dim \ker dU(y) > 0 
  \end{equation*}
  for all $y \in C$.  Thus, there is at least one nonvanishing, smooth vector field $v$ on $C$ that is annhiliated by $dU$.  Solving for the flow along $v$ yields a 1-parameter family of diffeomorphisms $D_t$.  The action of $D_t$ is a symmetry of $U$, namely for all $t\in \mathbb{R}$, $U \circ D_t = U$.
  Thus, $\phi: M \to C$ descends to the quotient $C / D$ -- whose dimension is strictly less than that of $C$ -- yielding a unique $\phi'$ making the diagram
  \begin{equation*}
    \xymatrix{
      M \ar[r]^{\phi} \ar[rd]_{\phi'}& C \ar[d] \ar[r]^{U} & N \\
      & C/D\ar[ru]_{U'} &\\
      }
  \end{equation*}
  commute.  Observe that $(\phi',U')$ is a quasiperiodic factorization, so we conclude that $(\phi,U)$ was not final in ${\bf QuasiP}(u)$ and therefore not a universal quasiperiodic factorization.
\end{proof}

\subsection{Noise performance}
\label{sec:noise_perf}
Performance of the QPLPF on noisy signals is governed both by Theorem \ref{thm:phase_recovery} and by the neighborhood size $S$.  We would like to minimize the recovery error in the $L^2$ norm,
\begin{align*}
  & \left\|(\text{QPLPF}\; \tilde{u})(x_i) - u(x_i) \right\| =  \left\| \frac{1}{1+ S}\sum_{[x_i,x_j]\in H} \tilde{u}(x_j) - u(x_i) \right\|\\
  & \le  \left\| \frac{1}{1+ S}\sum_{[x_i,x_j]\in H} u(x_j) - u(x_i)\right\| + \frac{\sigma}{\sqrt{1+S}}
\end{align*}
where we have used independence of the noise in the last step.  The first term above is the \emph{Stage 1 error} and the second term is the \emph{Stage 2 error}.  The Stage 2 error in the QPLPF is essentially the best that can be obtained without further knowledge of the statistics of $n$.

Given that $u$ is $(\phi,U)$-quasiperiodic, we can have substantially better control of the Stage 1 error.  Unless it is perfectly matched to the signal, a traditional filter has nonzero Stage 1 error even if there is no noise.  If there is no noise present and $S$ is small enough, so that
\begin{equation}
  \label{eq:neighborhood_size}
  S \le \# (\phi^{-1}(x_i) \cap X)
\end{equation}
for all $x_i \in X$, we have that $\phi(x_j) = \phi(x_i)$ for all adjacent pairs $(x_i,x_j)$.  This situation causes the Stage 1 error
\begin{align*}
  &\left\| \frac{1}{1+ S}\sum_{[x_i,x_j]\in H} u(x_j) - u(x_i)\right\| \\
  & = \left\| \frac{1}{1+ S}\sum_{[x_i,x_j]\in H} U\left(\phi(x_j)\right) - U\left(\phi(x_i)\right)\right\|
\end{align*}
to completely vanish for the QPLPF!  

\begin{prop}
When a $(\phi,U)$-quasiperiodic function $u$ with $\rank du(x) < \dim M$ for all $x$ is given as input to the QPLPF, the output is exactly $u$.
\end{prop}
\begin{proof}
The condition $\rank du(x) < \dim M$ ensures that preimages of points through $\phi$ have dimension greater than zero, so that \eqref{eq:neighborhood_size} can be satisfied.
\end{proof}


When noise is present, there is a tradeoff between keeping $S$ small enough to satisfy \eqref{eq:neighborhood_size} but large enough to control the Stage 2 error.  The Stage 1 error is controlled both by $S$ and $m$ through the construction of the graph $H$.  A loose upper bound on the Stage 1 error is
\begin{align*}
  &\left\| \frac{1}{1+ S}\sum_{[x_i,x_j]\in H} U\left(\phi(x_j)\right) - U\left(\phi(x_i)\right)\right\| \\
  & \le \|U\|_\infty \max_{[x_i,x_j]\in H} \left\| \phi(x_j) - \phi(x_i)\right\| 
\end{align*}
Although noise does not enter into the norm expression, it does impact our construction of $H$.  If the phase function (and hence $H$ also) is known outright, then $S$ can be chosen optimally even in the face of noise.  Otherwise, the QPLPF must rely on its estimate $F$ of $\phi$ instead.
\begin{align*}
  & \|U\|_\infty \max_{[x_i,x_j]\in H} \left\| \phi(x_j) - \phi(x_i)\right\| \\
  & \approx \|U\|_\infty \max_{[x_i,x_j]\in H} \left\| F(x_j) - F(x_i)\right\| \\
  & \le \|U\|_\infty \max_{[x_i,x_j]\in H} \left(\left\| u(x_j) - u(x_i)\right\| + \frac{\sigma}{\sqrt{m}}\right)
\end{align*}
Again, if $S$ is small enough, then all $[x_i,x_j]\in H$ will satisfy $u(x_i) \approx u(x_j)$, so first term above will typically be small.  The second term will usually dominate for small amounts of noise, and this can be controlled by increasing $m$.

\section{Results}
\label{sec:results}
This section presents three experimental data sets that validate both the theory and implementation of the QPLPF.  The first two data sets are simulated, while the third set uses image data collected by a satellite.

\subsection{Performance on simulated data}

\begin{figure}
\begin{center}
\includegraphics[width=3in]{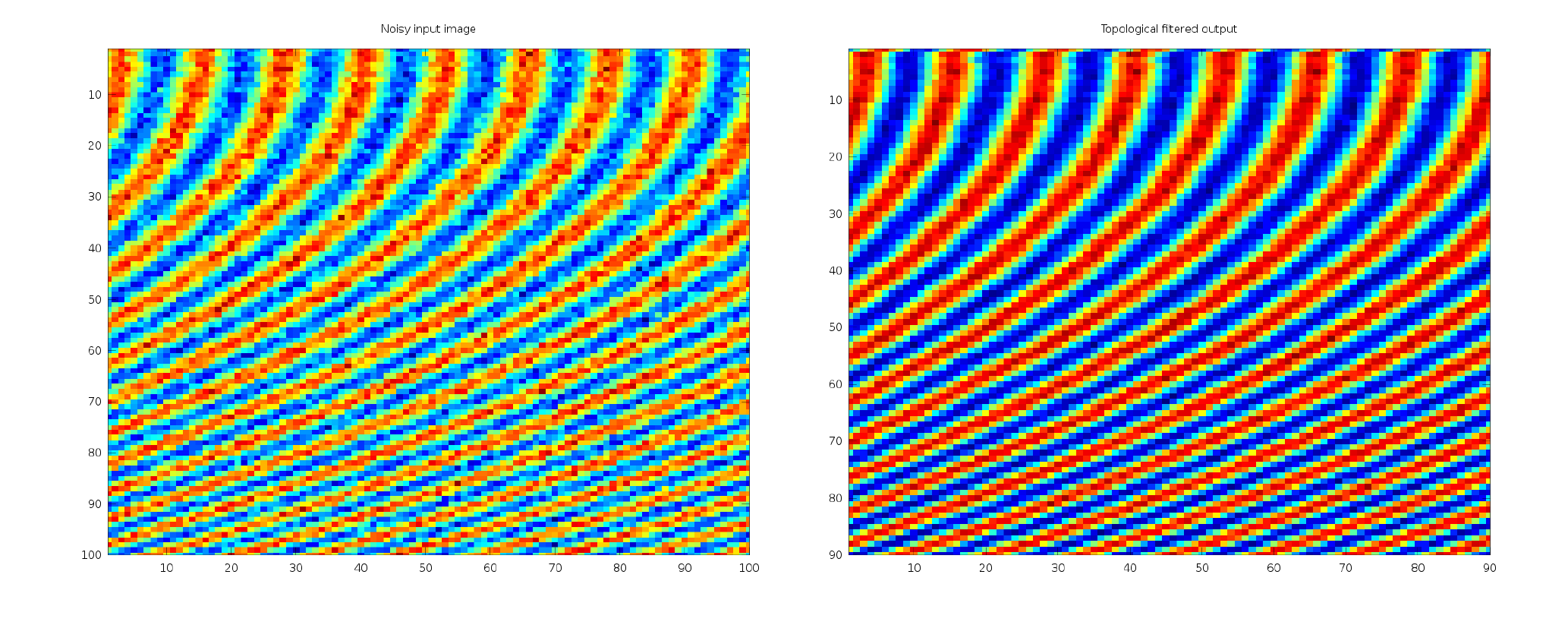} 
\caption{A noisy quasiperiodic image (left) and QPLPF output (right) when filtered with an matching window of 10 pixels and averaging 10 pixels.  Axes are in pixels.}
\label{fig:qpi_images}
\end{center}
\end{figure}

Figure \ref{fig:qpi_images} shows the performance of the QPLPF applied to a noisy quasiperiodic image (left).  The QPLPF output is shown at right, and shows a visible improvement over the entire image.

\begin{figure}
\begin{center}
\includegraphics[width=3in]{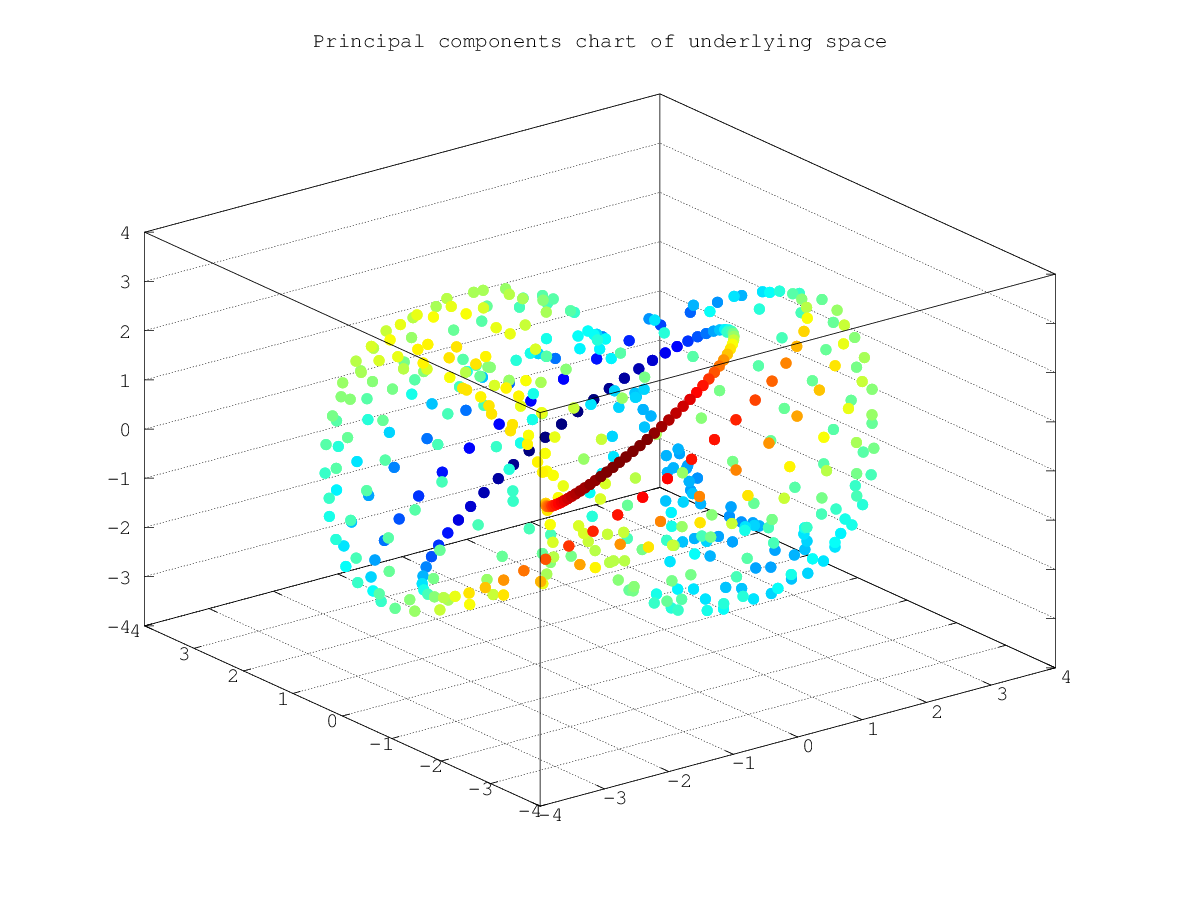} 
\caption{An example of the estimated space for the case of an LFM chirp.  Coordinates are the first four principal components ($x$, $y$, $z$, and color)}
\label{fig:lfm_space}
\end{center}
\end{figure}

Our implementation of the QPLPF on images is not particularly efficient, therefore for our statistical validation, we considered the discretized linear frequency modulated (LFM) chirp given by
\begin{equation}
  \label{eq:lfm}
u(t)=\sin\left(\frac{2\pi}{5} t (t+1)\right) + n(t)
\end{equation}
where $t=0, 1/50, \dotsc 10$ and $n(t)$ is additive white Gaussian noise.  This function is quasiperiodic, with a period that decreases with increasing $t$ over the given interval.  The output of Stage 1 of the QPLPF using a window size of (50 samples for the topological estimation stage and 15 samples for the averaging stage) is shown in Figure \ref{fig:lfm_space}, which suggests that the state space is a knotted circle.  The output of the QPLPF is shown as the red curve at right in Figure \ref{fig:lfm_output}.

\begin{figure}
\begin{center}
\includegraphics[width=3.5in]{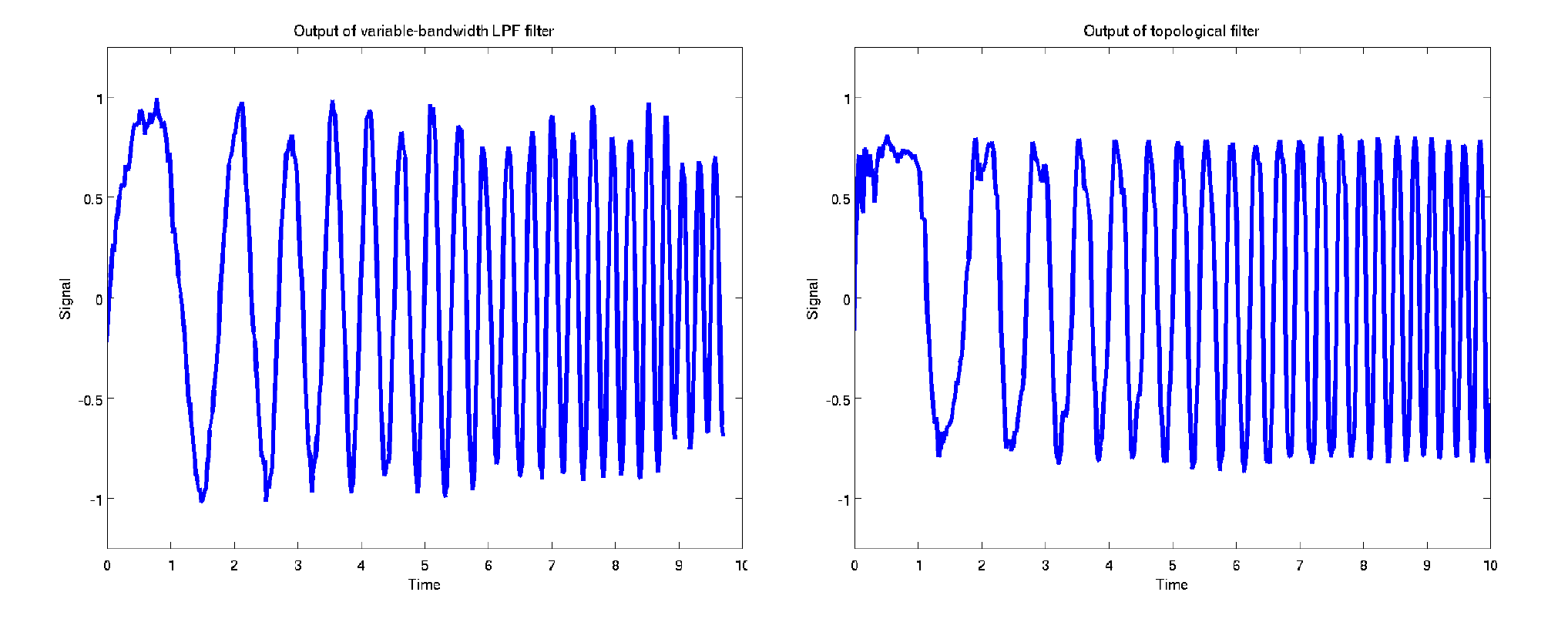} 
\caption{Comparison of output of a typical adaptive filter (left) and QPLPF (right) applied to a noisy LFM signal \eqref{eq:lfm}.  Both filters used a window size of 15 samples for averaging.  The QPLPF used a window size of 50 samples for topological estimation}
\label{fig:lfm_output}
\end{center}
\end{figure}

For comparison, the left frame of Figure \ref{fig:lfm_output} also shows the output of an adaptive variable-bandwidth filter, that uses as sliding window of 15 samples (same as the QPLPF) to estimate a local maximum frequency, and then sets the local averaging block size according to that frequency.  As the Figure shows, although the adaptive filter recovers the signal's frequency well, it does not produce a stable amplitude.  In contrast, the QPLPF does a better job of recovering the amplitude.  The QPLPF suffers no penalty as a function of SNR for this stability.

\begin{figure}
\begin{center}
\includegraphics[width=3.5in]{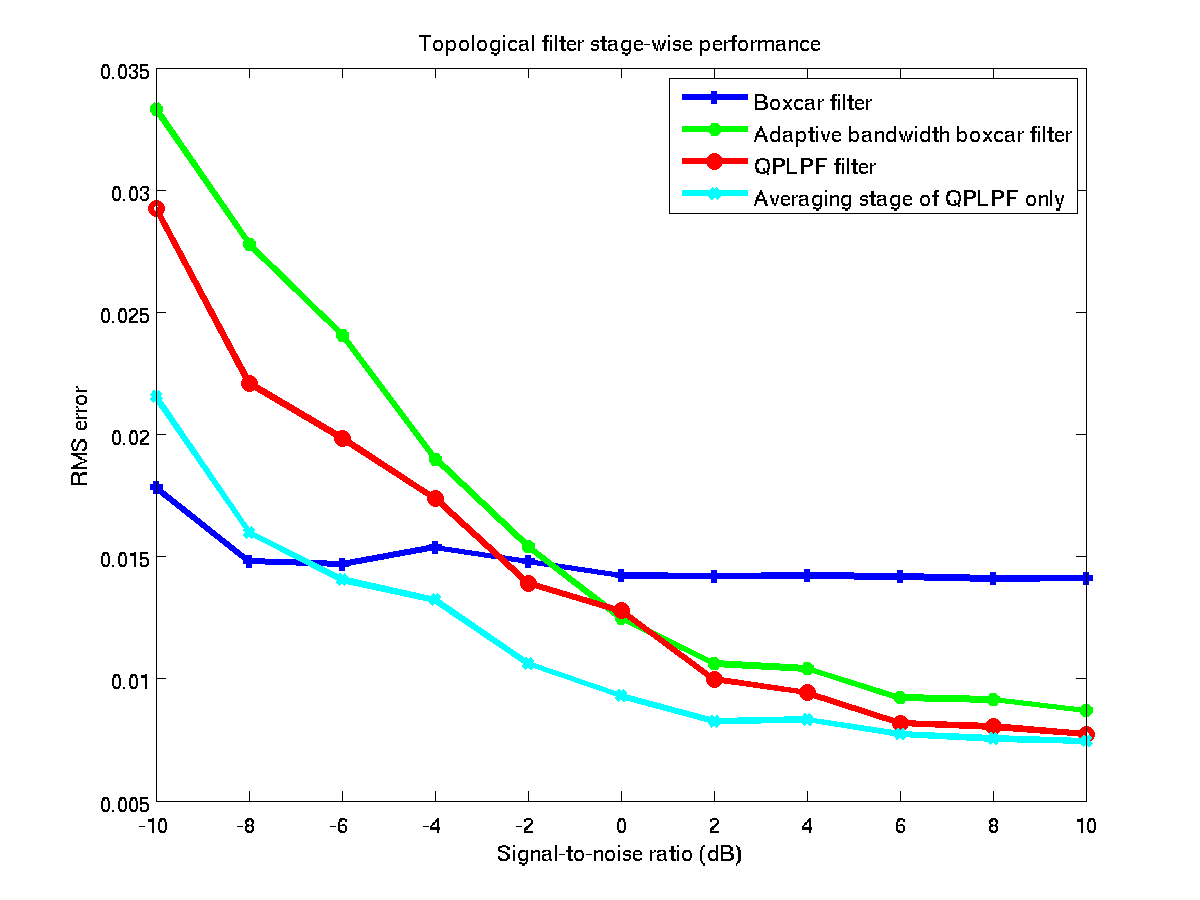} 
\caption{Comparison of the RMS filter error for a noisy LFM chirp as a function of SNR.  See text for window parameters.}
\label{fig:lfm_error_rms}
\end{center}
\end{figure}

\begin{figure}
\begin{center}
\includegraphics[width=3.5in]{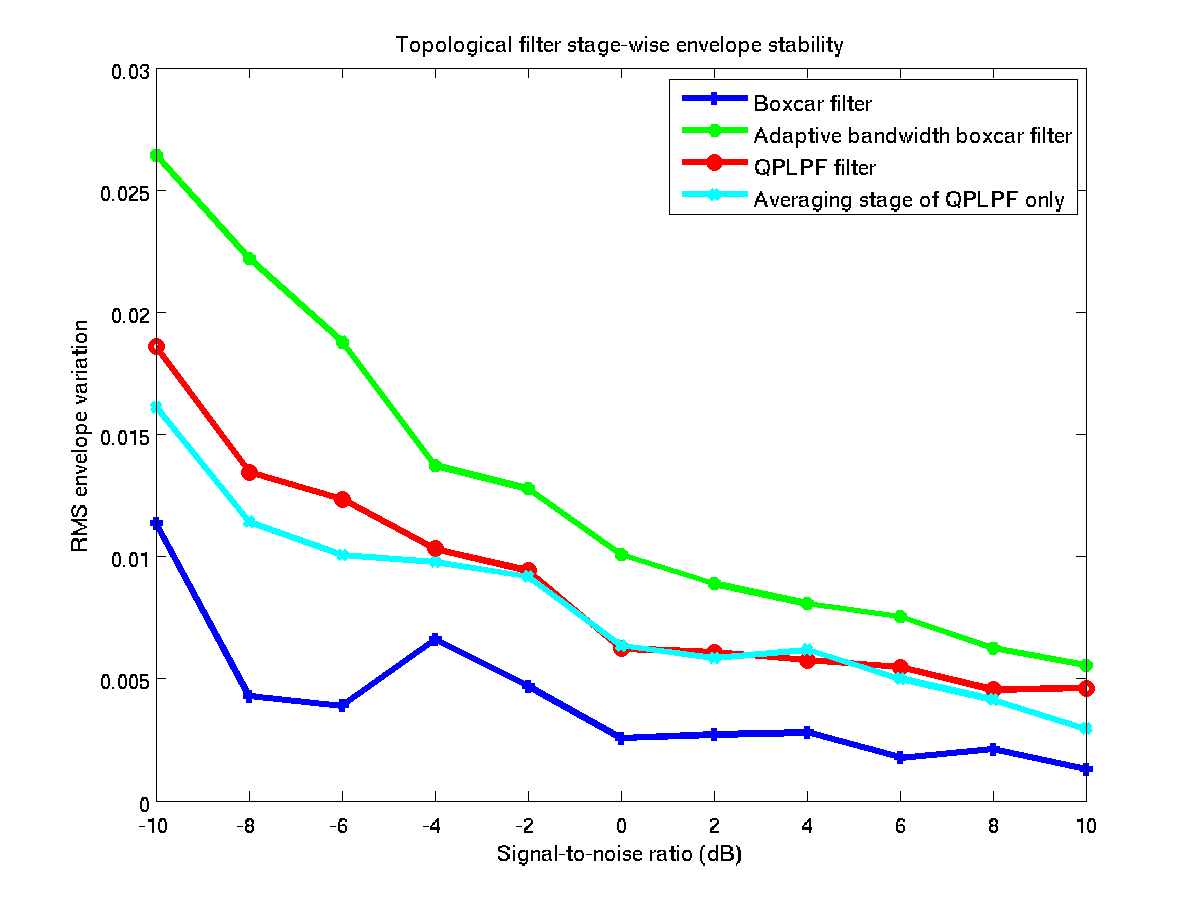} 
\caption{Comparison of the RMS envelope variability as a function of SNR.  See text for window parameters.}
\label{fig:lfm_error_env}
\end{center}
\end{figure}

Figures \ref{fig:lfm_error_rms} and \ref{fig:lfm_error_env} shows the performance of the QPLPF and the adaptive filter as a function of SNR for an LFM signal like what is shown in Figure \ref{fig:lfm_output}.  A boxcar filter and the averaging stage of the QPLPF using the true phase space -- both with a fixed window size of 11 samples -- are included for comparison.  The QPLPF used a window size of 50 samples for topological estimation and a window size of 11 samples for averaging. The adaptive boxcar filter used a window size of 50 samples for frequency estimation, and its averaging window was set adaptively at Nyquist based on this estimate.

The vertical axis of Figure \ref{fig:lfm_error_rms} shows the RMS difference between the original (noiseless) signal and the output of each filter.  Since the amplitude of the original signal was held constant at 1, the RMS measurement of the envelope of the ideal output should be zero.  The envelope signal is produced by linearly interpolating between peaks of the output signal.  The vertical axis of Figure \ref{fig:lfm_error_env} shows the RMS envelope of each output signal.  

The three variable-bandwidth filters (the QPLPF, the adaptive filter, and the QPLPF averaging stage) all exhibit improved RMS error and improved envelope stability as the SNR improves.  However, the QPLPF exhibits better performance when the SNR is lower.  The QPLPF exhibits considerably greater envelope stability than the adaptive filter, an effect which is most pronounced at low SNR.

\subsection{A maritime SAR image}

\begin{figure}
\begin{center}
\includegraphics[width=3in]{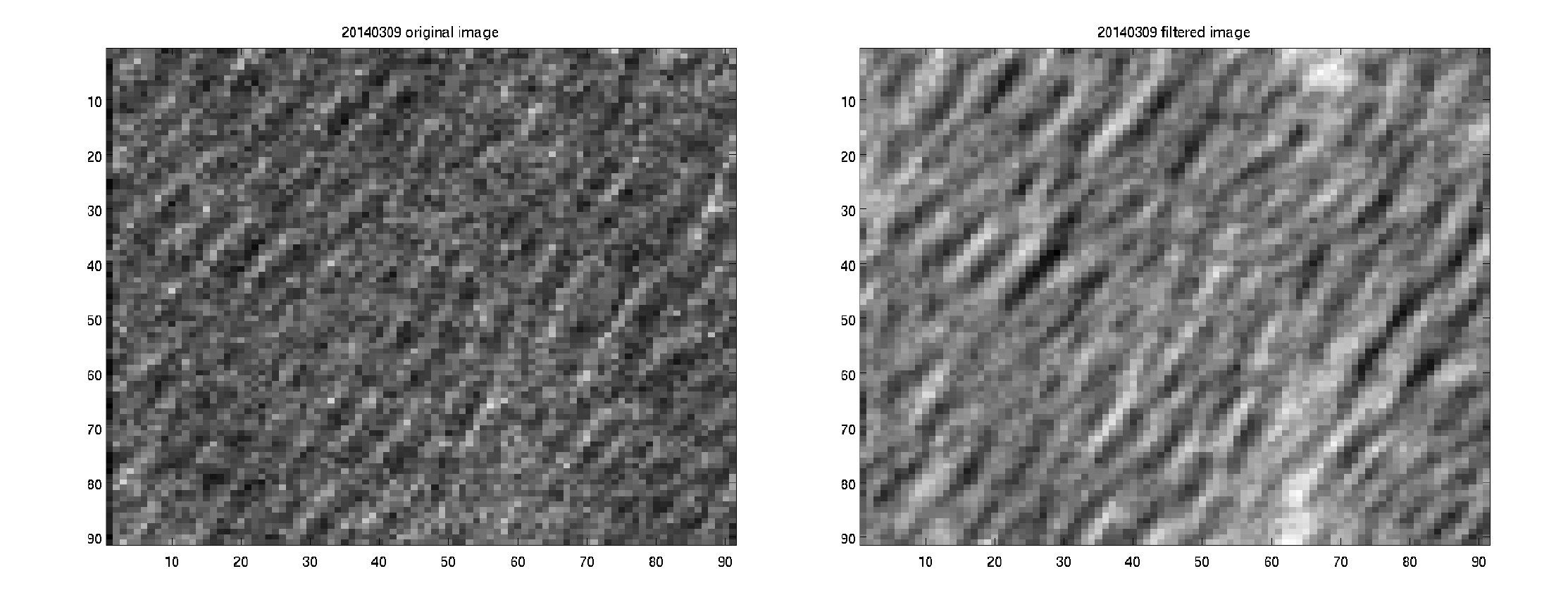} 
\caption{Ocean SAR image before (left) and after (right) the application of the QPLPF with topological estimation window of $10\times 10$ pixels and an averaging window of $150$ pixels.  Axes in pixels, each of which is a square, 25 meters on a side.  Images copyright \copyright DLR 2014.}
\label{fig:sar_image}
\end{center}
\end{figure}

\begin{figure}
\begin{center}
\includegraphics[width=3in]{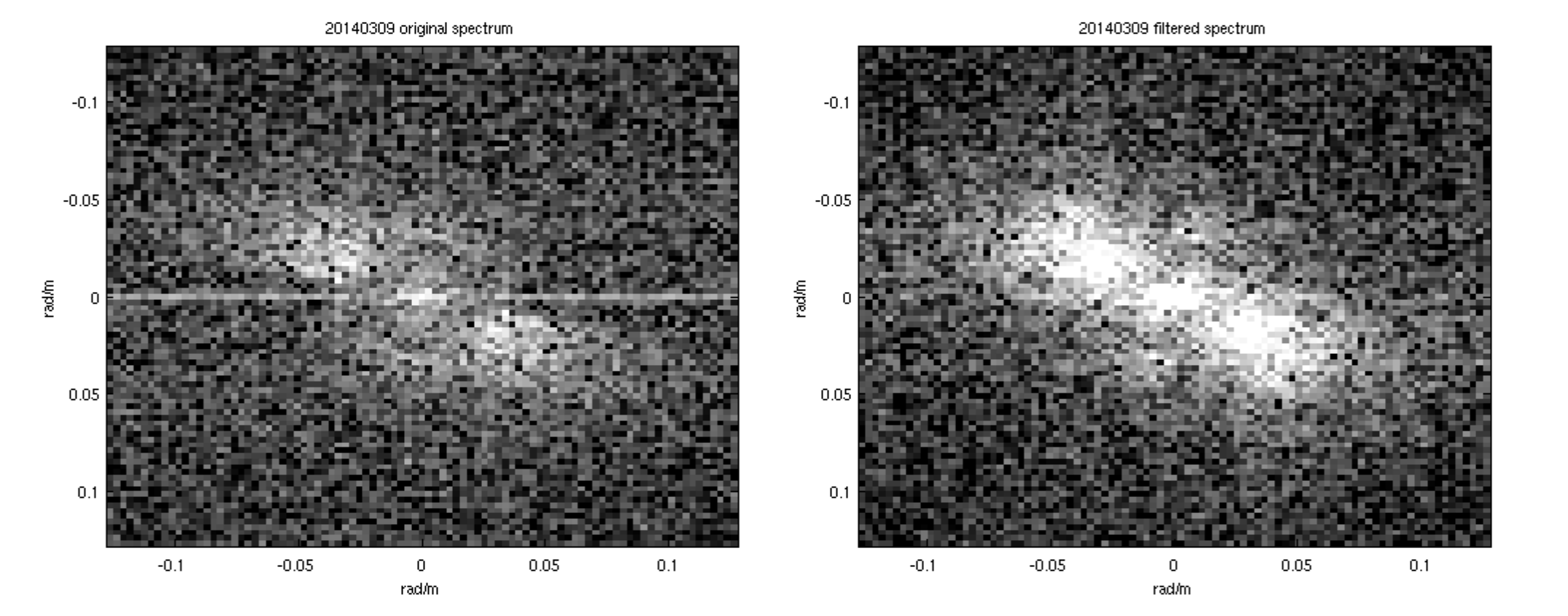} 
\caption{Spectrum of SAR images before (left) and after (right) the application of the QPLPF.  Axes in radian/meter}
\label{fig:sar_spectrum}
\end{center}
\end{figure}

This example demonstrates the QPLPF applied to the left frame of Figure \ref{fig:sar_image}, a $100\times 100$ pixel SAR image acquired by the German satellite TerraSAR-X on 9 March 2014 over the Gulf of Maine at 25 meters per pixel.  The diagonal striations in the image are produced by ocean swells that are roughly 80 meters in wavelength.  Figure \ref{fig:sar_spectrum} at left shows the 2d FFT of the image, from which the ocean wave spatial frequency and direction can be easily discerned.  Both the image and spectrum have been corrupted by speckle and noise, and the spectrum shows a horizontal streak artifact.  After applying the QPLPF with a matching window size of 10 pixels and a blocksize of 150 pixels, we obtain the images at right in Figures \ref{fig:sar_image} and \ref{fig:sar_spectrum}.  Notice that the QPLPF improves both the apparent contrast of individual waves and the SNR in the spectrum.

\section{Conclusion}

This article presented the QPLPF, a two-stage topological filter that performs averaging on an estimated phase space of a signal.  The correctness of this approach was proven theoretically, was demonstrated statistically on simulated data, and was exhibited on experimental data.


\section*{Acknowledgements}
The author would like to thank the American University Vice Provost for Graduate Studies and Research and the DC Space Grant Consortium for providing partial funding for this project.  Partial funding was also provided by the Office of Naval Research via Federal Contract No. N00014-15-1-2090.  The author also thanks the Deutsches Zentrum f\"{u}r Luft und Raumfahrt (DLR) for supplying the SAR imagery used on this project.

\bibliographystyle{IEEEbib}
\bibliography{qplpf_bib}

\end{document}